\documentclass[conference]{IEEEtran}
\IEEEoverridecommandlockouts
\usepackage{cite}
\usepackage{amsmath,amssymb,amsfonts}
\usepackage{algorithm}
\usepackage{algorithmic}
\usepackage{graphicx}
\usepackage{textcomp}
\usepackage{xcolor}
\usepackage{booktabs}
\usepackage{url}

\def\BibTeX{{\rm B\kern-.05em{\sc i\kern-.025em b}\kern-.08em
    T\kern-.1667em\lower.7ex\hbox{E}\kern-.125emX}}

\newcommand{\fbar}{\bar{f}}
\newcommand{\qt}{\tilde{q}}
\newcommand{\epst}{\tilde{\varepsilon}}
\newcommand{\norm}[1]{\lVert #1 \rVert}
\DeclareMathOperator*{\argmin}{arg\,min}
\DeclareMathOperator{\softmax}{softmax}

\begin{document}

\title{OT-CAL: Optimal-Transport-Guided Adaptive Sample Selection for Efficient Visual Learning}
\author{\IEEEauthorblockN{Suryansh Talukdar}
\IEEEauthorblockA{\textit{ Department of Computer Science} \\
\textit{Purdue University Fort Wayne}\\
Fort Wayne, Indiana, USA  \\
talusb01@pfw.edu}
\and
\IEEEauthorblockN{Jin Soung Yoo}
\IEEEauthorblockA{\textit{ Department of Computer Science} \\
\textit{Purdue University Fort Wayne}\\
Fort Wayne, Indiana, USA \\
jyoo@pfw.edu}

}

\maketitle

\begin{abstract}
Deep neural networks for image recognition are typically trained using mini-batches, with a forward pass to compute predictions and losses followed by a more computationally expensive backward pass to update model parameters. However, not all training samples contribute equally to learning, suggesting that computation can be reduced by performing backward passes only on informative samples. The challenge is that sample informativeness changes throughout training: difficult samples may be valuable early on, improving samples become important later, and class diversity becomes increasingly relevant toward the end. Existing methods often rely on fixed selection rules or computationally expensive per-sample gradient calculations. We propose OT-CAL, an adaptive data selection method that automatically adjusts its sampling preferences during training. OT-CAL represents each sample using five informativeness features, including difficulty, learnability, and uncertainty. A curriculum vector describes the currently preferred sample characteristics, and optimal-transport-based matching assigns selection weights according to how closely samples match this target. The curriculum vector is then updated using the selected samples, forming a feedback loop that adapts without a manually defined schedule. Experiments on three standard visual learning benchmarks show that, when selecting 50\% of each mini-batch, OT-CAL improves performance by 1.3--2.1 points over AdaSelection while achieving 1.18--1.24 times faster training than full-data training.
\end{abstract}


\section{Introduction}

Deep neural networks for visual learning are typically trained through iterative updates using mini-batches of images. At each iteration, a forward pass computes predictions and corresponding losses, while a backward pass computes gradients to update the model parameters. The change that a single image asks for is called that image's gradient. The backward pass usually costs about twice as much as the forward pass, and standard training performs it for every image at every step.

Yet not every image is equally useful at every moment. Partway through training, many images are already predicted correctly and confidently. Their loss is small, so their gradients barely change the weights, but they still cost as much to process as any other image \cite{paul2021,toneva2019}. This suggests an opportunity to reduce computation: perform the forward pass on the entire mini-batch, use the resulting predictions and losses to identify informative samples, and perform the backward pass only on those selected samples. This strategy is called online batch selection \cite{loshchilov2015,jiang2019}, where ``online'' means that the choice is made afresh at every training step. Keeping half of each mini-batch roughly halves the cost of the backward pass.

The central challenge is determining which samples to select, because their informativeness changes throughout training. At the start, difficult samples may provide strong learning signals because their losses are high. Once the network has learned the easy patterns, samples whose losses are still decreasing may become more informative, while already-mastered samples contribute relatively little and persistently difficult samples may provide limited benefit. Near the end, maintaining class diversity can help prevent the selection process from concentrating on a small set of difficult examples. A good selection rule should therefore change as training progresses.

The idea that the order in which a model sees its examples matters is the basis of curriculum learning \cite{bengio2009,kumar2010,wang2022}. \emph{Curriculum learning} organizes training examples according to a learning strategy, often adjusting their difficulty or importance as training progresses. In this paper, we use the word curriculum for any preference over training examples that changes during training.

Existing methods solve parts of this problem, but not all of it at low cost. They fall into three groups.

\textbf{Coreset methods} \cite{mirzasoleiman2020,killamsetty2021glister,killamsetty2021gradmatch} pick a small subset of images whose gradients, added together, closely match the gradients of the full dataset. Gradient-matching approaches can provide representative subsets, but computing and storing per-sample gradients introduces substantial computational and memory overhead. Some methods also require expensive pairwise comparisons. For a ResNet-50 network on the Tiny-ImageNet dataset (100,000 images), storing these gradients alone would take more than 100~GB.

\textbf{Score-based methods} \cite{loshchilov2015,mindermann2022,zhang2023,jiang2019} are much cheaper. They give each image a score, such as its loss, or combine several such scores, and keep the highest-scoring images. However, methods based on fixed scoring rules may not adequately account for changes in sample informativeness during training.

\textbf{Curriculum methods} do change their preference over time, but something has to decide when to move from one stage to the next. This is usually a timetable written by hand \cite{hacohen2019}, a second ``teacher'' network trained to choose the examples \cite{fan2018,jiang2018}, or a separately learned controller \cite{graves2017}. These approaches may require predefined schedules, auxiliary models, or additional tuning across datasets.

These limitations motivate a computationally efficient selection method that adapts to the model's evolving learning needs without requiring a manually specified curriculum schedule.

This paper presents OT-CAL (Optimal-Transport Curriculum Adaptive Learning), which meets this need by letting the selection steer itself. It has three parts.

\textbf{Part 1: Sample informativeness profiles.} Each image is described by five scores, all computed from the forward pass without requiring per-sample backward passes:
\begin{itemize}
\item \textbf{Difficulty:} how large the image's loss is compared with the other images in the mini-batch.
\item \textbf{Learnability:} how much the image's loss has fallen compared with its recent past.
\item \textbf{Gradient proxy:} how far the network's predicted probabilities are from the correct answer. This gives a cheap estimate of how large the image's gradient is, without computing the gradient itself.
\item \textbf{Uncertainty:} how unsure the network is about its prediction.
\item \textbf{Variety:} how rare the image's class is in the current mini-batch.
\end{itemize}
Together, these five features form a sample informativeness \emph{profile}. No gradients for individual images are computed or stored.

\textbf{Part 2: Adaptive curriculum target.} OT-CAL keeps one more profile, called the curriculum vector. It has the same five numbers, but instead of describing a particular image, it describes the preferred characteristics of training samples at the current stage of learning. At each step, the images whose profiles best match the curriculum vector are kept for the backward pass.

\textbf{Part 3: Feedback-based curriculum update.} After each step, the curriculum vector moves a small step toward the weighted average profile of the images just selected: the new curriculum vector is 95\% the old one plus 5\% the new average. This kind of update is called an exponential moving average (EMA). Because each update is small, the target changes gradually and is not thrown off by any single mini-batch.

Together, these parts form a feedback loop: the curriculum vector decides which images are selected, and the selected images then move the curriculum vector. As the network learns, sample profiles evolve with changes in prediction errors, learnability, and uncertainty. The curriculum vector follows these changes, allowing selection preferences to adapt throughout training.

A key design question is how to match samples to the curriculum target and determine their influence on subsequent curriculum updates. The obvious way is to sort the images by their distance from the curriculum vector, keep the nearest half, and give each of them an equal say in the update. However, uniform weighting can make the feedback sensitive to changes near the selection threshold. An image just inside the cut-off is fully kept and one just outside is fully dropped, so a tiny change in the scores can swap images in and out. Each swap changes the average profile, which moves the target, which can cause more swaps, potentially introducing abrupt changes in the curriculum update.

OT-CAL addresses this issue using entropy-regularized optimal transport (OT), a mathematical framework for matching distributions according to a transportation cost \cite{cuturi2013,peyre2019}. In OT-CAL, the cost measures the distance between each sample's informativeness profile and the curriculum target. The resulting matching weights quantify how closely samples match the target. Images whose profiles fit the target well receive more weight.

Entropy regularization encourages a smoother distribution of matching weights, reducing their sensitivity to small changes in sample profiles or the curriculum target. OT-CAL keeps the images with the largest weights and uses the same weights when it updates the target, so images near the cut-off, which carry little weight, barely move it. In the single-target formulation, selecting the largest matching weights is equivalent to selecting the samples nearest to the target. The benefit of OT-based weighting therefore lies primarily in how the selected samples influence the subsequent curriculum update, rather than in changing their ranking. Computing the weights takes about 1.2~ms per mini-batch of 128 images (Section~\ref{sec:tradeoff}).

Because OT-CAL decides what the network learns from, two questions naturally arise. First, how closely does the gradient computed from selected samples approximate the full mini-batch gradient? Second, under what conditions does the curriculum feedback remain stable? Section~\ref{sec:theory} answers both. First, the matching weights change smoothly: a small change in the scores or the target causes only a small change in the weights. Second, the gap between the combined gradient of the selected images and that of the full mini-batch is bounded: it depends on how well the five scores reflect each image's actual gradient and on how different the selected images' scores are from the batch's, and it shrinks as more of the batch is kept. Third, we establish a sufficient condition for contraction of the curriculum update, while recognizing that this condition is not guaranteed for all settings. The second result suggests that better scores should give better training. We test this with a variant in which a small neural network learns an eight-number profile that tracks gradients more closely; this variant gives our best results.

We test OT-CAL on three standard benchmarks: CIFAR-100 and Tiny-ImageNet, for image classification, and COCO, for object detection. We use two widely used classification networks, ResNet and the Vision Transformer (ViT), and the DETR object detector. At a 50\% sampling ratio, OT-CAL improves performance over AdaSelection \cite{zhang2023} by 1.3--2.1 percentage points across the evaluated settings and achieves $1.18$--$1.24\times$ faster training than full-data training. For CIFAR-100, the reported accuracy uses the learned 8-D profile, whereas the speedup is measured using the fixed 5-D profile. (On COCO, accuracy is measured by mean average precision, the standard score for object detection.) Ablation studies show that learnability contributes most strongly among the five informativeness features, while OT-based weighting provides modest but consistent improvements over uniform top-$k$ selection using the same profiles.

In summary, this paper makes three contributions:
\begin{itemize}
\item \textbf{Adaptive data selection:} We propose OT-CAL, which integrates curriculum learning and online batch selection through an evolving curriculum target. The selection preferences adapt to training dynamics without requiring a predefined curriculum schedule or per-sample gradient computation.
\item \textbf{Theoretical analysis:} We analyze the stability of transport-based matching weights, establish a bound on the gradient bias introduced by subset selection, and characterize conditions under which the curriculum update is contractive.
\item \textbf{Experimental evaluation:} We evaluate OT-CAL on three visual learning benchmarks using multiple model architectures, demonstrating improved accuracy--efficiency tradeoffs over existing selection methods. Ablation studies examine the contributions of individual informativeness features, learned representations, and transport-based weighting.
\end{itemize}

\section{Related Work}

\subsection{Coreset and Online Batch Selection}
Coresets are weighted subsets that provably approximate full-data objectives \cite{feldman2011,bachem2017}. In deep learning, CRAIG \cite{mirzasoleiman2020} greedily covers the gradient space, GLISTER \cite{killamsetty2021glister} selects subsets through a bilevel generalization objective, and GRAD-MATCH \cite{killamsetty2021gradmatch} matches the full-data gradient with orthogonal matching pursuit. Geometric coresets \cite{sener2018} and training-dynamics scores such as EL2N, forgetting events, and data maps \cite{paul2021,toneva2019,swayamdipta2020} prune data before or early in training; DeepCore \cite{guo2022} compares many of these methods. Online methods instead select within each mini-batch: selective backpropagation skips the backward pass for low-loss samples \cite{jiang2019}, importance sampling reduces gradient variance \cite{loshchilov2015,katharopoulos2018}, RHO-LOSS selects points that are learnable and not yet learned using a holdout model \cite{mindermann2022}, and AdaSelection aggregates several scoring rules \cite{zhang2023}. Gradient-matching methods need per-sample gradients and pairwise computation, while score-aggregation methods use weights that do not adapt to the phase of training.

\subsection{Curriculum Learning}
Curriculum learning \cite{bengio2009} and self-paced learning \cite{kumar2010} present examples in a structured order to stabilize optimization; \cite{wang2022} surveys the area. Schedules may be predefined \cite{hacohen2019}, produced by a teacher network \cite{fan2018,jiang2018}, or chosen by a bandit over tasks \cite{graves2017}. These approaches require either domain knowledge to set pacing functions or an auxiliary learner, and convergence guarantees for curriculum-driven sampling remain limited.

\subsection{Optimal Transport in Machine Learning}
Optimal transport compares probability measures through the minimal cost of moving mass between them \cite{peyre2019}. Entropic regularization and the Sinkhorn algorithm \cite{cuturi2013} make OT fast and differentiable, and the entropic solution varies smoothly with its inputs \cite{genevay2016,luise2018}. OT has been applied to domain adaptation \cite{courty2017,damodaran2018,shen2018} and mini-batch alignment \cite{fatras2021}, and gradient matching has been used for dataset condensation \cite{zhao2021}. To our knowledge, OT has not been used to couple an evolving curriculum with online subset selection.

\subsection{Positioning}
OT-CAL differs from prior work in three respects: (i) it formulates subset selection as distributional matching in a compact profile space, unifying curriculum and coreset selection; (ii) the target distribution evolves through closed-loop EMA feedback, so phases need not be scheduled; and (iii) it pairs $O(N\log N)$ per-batch cost with explicit stability and gradient-bias guarantees.

\section{Method}\label{sec:method}

\subsection{Overview}\label{sec:overview}
OT-CAL repeats a select--train--update loop at every training iteration $t$ (Fig.~\ref{fig:loop}, Algorithm~\ref{alg:otcal}). A curriculum vector $c_t$ describes the kind of sample the model should learn from at step $t$. (i) \emph{Select:} the forward pass yields an informativeness profile $f_i$ for every sample in the mini-batch (Section~\ref{sec:profile}), and entropic OT matches these profiles to $c_t$ and keeps the $k$ best-matched samples, $S_t$ (Section~\ref{sec:ot}). (ii) \emph{Train:} the backward pass and weight update use only $S_t$. (iii) \emph{Update:} the curriculum moves a small step toward the profiles that were just selected:
\begin{equation}
c_{t+1} = (1-\eta_c)\,c_t + \eta_c\,\fbar_t,\label{eq:ema}
\end{equation}
where $\fbar_t = \sum_{i\in S_t} \qt_i f_i$ is the mean profile of the selected samples, weighted by their normalized matching weights $\qt_i$ (Section~\ref{sec:ot}), and $\eta_c = 0.05$ is the curriculum update rate. Update \eqref{eq:ema} is an exponential moving average (EMA): each new curriculum is 95\% the previous one and 5\% the latest selection. It is the feedback mechanism of OT-CAL: $c_t$ decides which samples are selected, and the selected samples move $c_t$. Because the profiles change as the model learns, $c_t$ keeps moving, so the preference over samples changes during training without an epoch-based schedule. The curriculum is initialized from the data as the mean profile of the first mini-batch, $c_0 = \frac{1}{N}\sum_{i=1}^{N} f_i$.

\begin{figure}[t]
\centerline{\includegraphics[width=0.95\columnwidth]{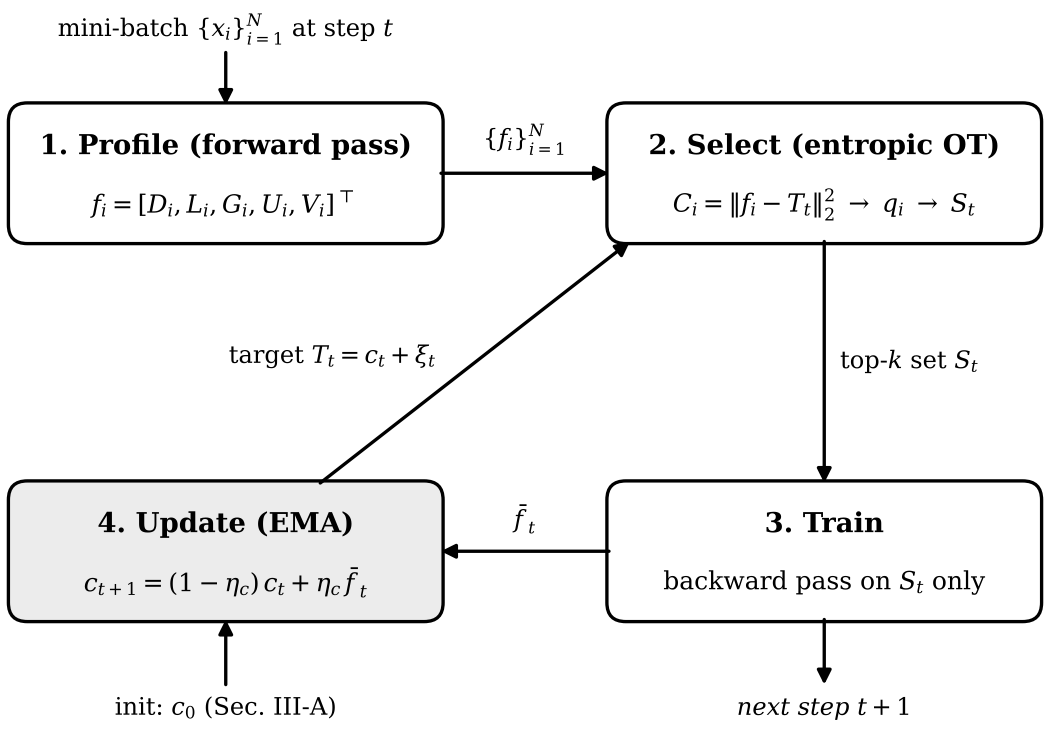}}
\caption{The OT-CAL loop at iteration $t$. The curriculum vector $c_t$ sets the target for selection; the selected subset $S_t$ is used for the backward pass and then moves $c_t$ through the EMA update \eqref{eq:ema}.}
\label{fig:loop}
\end{figure}

\subsection{Sample Informativeness Profile}\label{sec:profile}
Each sample $i$ is described by a five-dimensional profile $f_i = [D_i, L_i, G_i, U_i, V_i]^\top$, computed from the forward pass only. Let $\ell_i$ be the loss of sample $i$, $p_i$ its predicted distribution over $K$ classes, and $y_i$ its one-hot label.
\begin{itemize}
\item \emph{Difficulty} $D_i = \ell_i / \max_j \ell_j$: the loss relative to the largest loss in the batch. High-loss samples give strong gradients early in training \cite{bengio2009}.
\item \emph{Learnability} $L_i = \tanh\!\big((\ell_i^{\mathrm{EMA}} - \ell_i)/\ell_i^{\mathrm{EMA}}\big)$: the relative loss improvement over $\ell_i^{\mathrm{EMA}}$, an EMA of the losses of sample $i$ on its earlier visits (decay 0.9 per visit, initialized to the first observed loss, so $L_i = 0$ on a first visit). It separates samples that are still improving from stagnant or already-learned ones \cite{toneva2019,mindermann2022}.
\item \emph{Gradient proxy} $G_i = \norm{p_i - y_i}_2$: the norm of the cross-entropy gradient with respect to the logits, a cheap stand-in for the per-sample gradient norm \cite{paul2021,katharopoulos2018}.
\item \emph{Uncertainty} $U_i = -\sum_{m=1}^{K} p_{im}\log p_{im}$: the predictive entropy, which is high for samples near decision boundaries.
\item \emph{Variety} $V_i = 1/n_i^{B}$, where $n_i^{B}$ is the number of samples of the same class as $i$ in the current mini-batch: it favors classes that are under-represented in the batch. A dataset-level class frequency would be constant here, because CIFAR-100 and Tiny-ImageNet are exactly class-balanced; at batch size 128 on CIFAR-100, a sample's class appears 1, 2, 3, or $\geq 4$ times in its batch for about 28\%, 36\%, 23\%, and 13\% of samples.
\end{itemize}

\emph{Scaling.} The raw features have different ranges (e.g., $U_i \in [0,\log K]$ but $D_i\in[0,1]$), so each is rescaled to $[0,1]$ before matching; otherwise the widest-range feature would dominate the cost in Section~\ref{sec:ot}. $L_i \in (-1,1)$ is mapped to $(L_i+1)/2$, $G_i$ is divided by its maximum $\sqrt{2}$, $U_i$ by $\log K$, and $V_i$ is min--max scaled within the batch.

\emph{Detection.} For DETR, each image is one sample. $\ell_i$ is its Hungarian-matched set-prediction loss, from which $D_i$ and $L_i$ follow as above; $G_i$ and $U_i$ average the gradient proxy and the normalized entropy over the queries matched to ground-truth objects; and $V_i$ averages the batch inverse class frequency over the image's objects.

Each feature has individual precedent \cite{paul2021,toneva2019,mindermann2022}; the contribution of OT-CAL is to combine them into a target that moves during training. The curriculum vector lives in the same space, $c_t = [D_t^{*}, L_t^{*}, G_t^{*}, U_t^{*}, V_t^{*}]^\top$, where each entry is the preferred level of that feature at step $t$. Learnability needs one loss-EMA scalar per training example; no per-sample gradients are stored.

\emph{Learned 8-D profile.} Theorem~\ref{thm:bias} (Section~\ref{sec:theory}) predicts that selection is more accurate when distances between profiles track distances between gradients. We therefore also use a learned variant. A two-layer MLP $g:\mathbb{R}^5\to\mathbb{R}^8$ (hidden width 64, ReLU) maps the scaled profile to an embedding $z_i = g(f_i)$, and a linear head $h:\mathbb{R}^8\to\mathbb{R}$ predicts the per-sample gradient norm from $z_i$. The regression targets are last-layer gradient norms, which have the closed form $\norm{\nabla_W \ell_i} = \norm{p_i - y_i}_2\,\norm{a_i}_2$, with $a_i$ the penultimate-layer activation, so they come from the forward pass without a per-sample backward pass. $g$ and $h$ are trained online with their own Adam optimizer (learning rate $10^{-3}$), one MSE step per mini-batch, with gradients stopped at the network. In this variant, $z_i$ replaces $f_i$ everywhere: the cost in \eqref{eq:cost}, the curriculum vector $c_t\in\mathbb{R}^8$ (initialized to the first-batch mean of $z_i$), and update \eqref{eq:ema}; nothing else changes.

\subsection{OT-Based Sample Selection}\label{sec:ot}
Selection compares each profile with a perturbed target $T_t = c_t + \xi_t$, where $\xi_t \sim \mathcal{N}(0, \sigma^2 I)$ and $\sigma = 0.05$; the noise keeps the loop from locking onto a fixed point. The matching cost of sample $i$ is its squared distance from the target:
\begin{equation}
C_i = \norm{f_i - T_t}_2^2,\qquad i = 1,\dots,N.\label{eq:cost}
\end{equation}
The mini-batch is the source distribution, with mass $a_i = 1/N$ on each of its $N$ samples, and the target is a single point $T_t$ with mass $b = 1$. A transport plan $P\in\mathbb{R}_+^N$ sends mass $P_i$ from sample $i$ to the target, and we solve
\begin{equation}
P^{*} = \argmin_{P\in\Pi_b}\; \langle P, C\rangle - \varepsilon H(P) + \tau\,\mathrm{KL}(P\,\|\,a),\label{eq:ot}
\end{equation}
where $\Pi_b = \{P\in\mathbb{R}_+^N : \sum_i P_i = b\}$ fixes only the target mass, $\langle P, C\rangle = \sum_i P_i C_i$ is the total transport cost, $H(P) = -\sum_i P_i \log P_i$ is the entropy of the plan, $\mathrm{KL}(P\,\|\,a) = \sum_i \big(P_i\log(P_i/a_i) - P_i + a_i\big)$ penalizes deviation from the source marginal, $\varepsilon = 0.01$, and $\tau > 0$.

\emph{Why the source marginal is relaxed.} If the plan had to match both marginals exactly (balanced OT), the only feasible plan for a single target would be $P_i = a_i = 1/N$ for every $i$, whatever the costs, and the plan could not rank samples. Problem \eqref{eq:ot} therefore keeps the target constraint $\sum_i P_i = b = 1$ and replaces the source constraint by the penalty $\tau\,\mathrm{KL}(P\,\|\,a)$, as in unbalanced OT \cite{chizat2018}. A sample that is expensive to move can then send less than $1/N$ of mass. Setting the derivative of \eqref{eq:ot} to zero gives $P_i \propto a_i^{\tau/(\varepsilon+\tau)}\exp\!\big(-C_i/(\varepsilon+\tau)\big)$; with uniform $a$ the first factor is constant, so the optimal plan has the closed form
\begin{equation}
q_i := P^{*}_i = \frac{\exp(-C_i/\epst)}{\sum_j \exp(-C_j/\epst)},\qquad \epst = \varepsilon + \tau.\label{eq:closed}
\end{equation}
With a single target point, unbalanced Sinkhorn scaling \cite{cuturi2013,chizat2018} reaches \eqref{eq:closed} after its first iteration. Our implementation runs a generic Sinkhorn solver for a fixed 50 iterations, which returns exactly the weights in \eqref{eq:closed}; the timings in Section~\ref{sec:tradeoff} are for this solver. Evaluating \eqref{eq:closed} directly as a softmax gives the same weights at lower cost. OT-CAL keeps the $k = \lceil \gamma N\rceil$ samples with the largest weights, $S_t = \operatorname{top-}k(q)$, where $\gamma\in(0,1]$ is the sampling ratio, and normalizes their weights, $\qt_i = q_i / \sum_{j\in S_t} q_j$, for update \eqref{eq:ema}. We use $\tau = 0.1$, so $\epst = 0.11$: costs on $[0,1]^5$ typically differ by 0.1--1, so $\varepsilon = 0.01$ alone would place almost all weight on one or two samples.

\emph{What OT adds over top-$k$ ranking.} With one target, $q_i$ decreases with $C_i$, so $S_t$ holds the same samples as the $k$ nearest neighbors of $T_t$. The difference is in how the selection feeds back into the curriculum. Plain top-$k$ gives every selected sample equal influence on $\fbar_t$, so a sample that crosses the rank-$k$ threshold changes $\fbar_t$ abruptly, and \eqref{eq:ema} passes that jump on to the next selection. Under entropic OT, samples near the threshold carry smaller weights $\qt_i$, so a crossing moves $\fbar_t$ less, and the weights themselves vary smoothly with the profiles and the target (Proposition~\ref{prop:stab}). In our ablation (Section~\ref{sec:ablation}), OT selection improves accuracy over top-$k$ on identical features by 0.2--0.3 points and increases selection stability. Table~\ref{tab:notation} summarizes the notation.

\begin{table}[t]
\caption{Notation}
\label{tab:notation}
\centering
\footnotesize
\begin{tabular}{@{}p{0.22\columnwidth}p{0.7\columnwidth}@{}}
\toprule
\textbf{Symbol} & \textbf{Meaning}\\
\midrule
$N$, $K$ & Mini-batch size (128 for classification); number of classes\\
$f_i$ & Profile of sample $i$ (5-D; learned 8-D variant $z_i$)\\
$c_t$ & Curriculum vector: preferred profile at iteration $t$\\
$c_0$ & Initial curriculum: mean profile of the first mini-batch\\
$T_t$, $\sigma$ & Perturbed target $c_t + \xi_t$; noise std.\ (0.05)\\
$C_i$ & Matching cost of sample $i$, $\norm{f_i - T_t}_2^2$\\
$a$, $b$ & Source marginal ($a_i = 1/N$); target mass ($b = 1$)\\
$P$, $P^{*}$ & Transport plan; optimal plan of \eqref{eq:ot}\\
$\langle P, C\rangle$ & Total transport cost $\sum_i P_i C_i$\\
$H(P)$ & Entropy of the plan; larger means more spread out\\
$\Pi_b$ & Admissible plans: target mass $b$ exact; source relaxed by $\tau\,\mathrm{KL}(P\,\|\,a)$\\
$\varepsilon$, $\tau$ & Entropy weight (0.01); source-relaxation weight (0.1)\\
$q_i$, $\qt_i$ & Matching weight $P^{*}_i$; normalized over $S_t$\\
$\gamma$, $k$ & Sampling ratio; number selected, $k = \lceil\gamma N\rceil$\\
$S_t$ & Samples selected for training at iteration $t$\\
$\fbar_t$ & Weighted mean profile of $S_t$\\
$\eta_c$, $\eta$ & Curriculum update rate (0.05); network learning rate\\
\bottomrule
\end{tabular}
\end{table}

\subsection{Curriculum Update and Hyperparameters}\label{sec:update}
After the backward pass on $S_t$, the curriculum is updated by \eqref{eq:ema}. Because Algorithm~\ref{alg:otcal} applies \eqref{eq:ema} once per mini-batch, $\eta_c = 0.05$ gives a time constant of $1/\eta_c = 20$ iterations, about 0.05 epoch on CIFAR-100 (391 iterations per epoch at batch size 128). The curriculum therefore tracks the current profiles closely; its slows evolution over training (Section~\ref{sec:dynamics}) reflects how the profiles themselves change, not lag in the EMA. A single decay parameter avoids the tuning burden of more elaborate controllers such as PID or momentum updates. The OT-based selection and curriculum update have five hyperparameters ($\gamma$, $\varepsilon$, $\tau$, $\eta_c$, $\sigma$), and one setting is used for all experiments. The optional learned 8-D profile (Section~\ref{sec:profile}) adds only the training settings of its MLP, which are fixed as stated there.

\begin{algorithm}[t]
\caption{OT-CAL per-batch selection and curriculum update}
\label{alg:otcal}
\begin{algorithmic}[1]
\REQUIRE batch $\{x_i\}_{i=1}^N$, curriculum $c_t$, sampling ratio $\gamma$, $\varepsilon = 0.01$, $\tau = 0.1$, $\eta_c = 0.05$, $\sigma = 0.05$
\ENSURE subset $S_t$, curriculum $c_{t+1}$
\FOR{each sample $i$}
\STATE compute scaled $f_i = [D_i, L_i, G_i, U_i, V_i]^\top$ \COMMENT{forward pass only}
\ENDFOR
\STATE Sample $\xi_t \sim \mathcal{N}(0, \sigma^2 I)$; $T_t \leftarrow c_t + \xi_t$
\STATE $C_i \leftarrow \norm{f_i - T_t}_2^2$ for all $i$
\STATE $q \leftarrow \textsc{Sinkhorn}(C, \varepsilon, \tau; 50\text{ iter.})$ \COMMENT{returns \eqref{eq:closed}}
\STATE $k \leftarrow \lceil\gamma N\rceil$; $S_t \leftarrow$ indices of the $k$ largest $q_i$; $\qt_i \leftarrow q_i / \sum_{j\in S_t} q_j$
\STATE $w \leftarrow w - \eta\,\nabla_w \frac{1}{k}\sum_{i\in S_t}\ell(x_i; w)$ \COMMENT{backward on $S_t$ only}
\STATE $\fbar_t \leftarrow \sum_{i\in S_t} \qt_i f_i$
\STATE $c_{t+1} \leftarrow (1-\eta_c)\,c_t + \eta_c\,\fbar_t$
\RETURN $S_t$, $c_{t+1}$
\end{algorithmic}
\end{algorithm}

\section{Theoretical Analysis}\label{sec:theory}

We ground OT-CAL in entropic OT and biased-SGD theory. The individual tools are standard \cite{cuturi2013,genevay2016,ghadimi2013,ajalloeian2020}; the contribution is their combination for a time-varying target under closed-loop feedback. Throughout, profiles lie in $[0,1]^d$ ($d = 5$, or 8 for the learned variant). Proofs are given in the Appendix.

\newtheorem{assumption}{Assumption}
\newtheorem{proposition}{Proposition}
\newtheorem{theorem}{Theorem}
\newtheorem{corollary}{Corollary}

\begin{assumption}[Feature informativeness]\label{ass:feat}
At the current weights $w$, for all samples $i,j$: $\norm{\nabla\ell_i - \nabla\ell_j}_2 \leq C_f\norm{f_i - f_j}_2 + r$, with constant $C_f$ and residual $r$. The loss is $L$-smooth.
\end{assumption}

\begin{proposition}[Stability of the matching weights]\label{prop:stab}
Let $F, F'\in[0,1]^{N\times d}$ be profile matrices and $T, T'\in[0,1]^d$ targets. The weights \eqref{eq:closed} satisfy
\[
\norm{q(F,T) - q(F',T')}_1 \leq L_{\mathrm{OT}}\big(\norm{F - F'}_F + \norm{T - T'}_2\big),
\]
with $L_{\mathrm{OT}} = 4\sqrt{d}/\epst$.
\end{proposition}
\begin{IEEEproof}
See Appendix, part 1.
\end{IEEEproof}

Proposition~\ref{prop:stab} implies that the soft weights, and hence any weighted statistic of the batch, vary continuously as the features and the curriculum drift. This is the property that the feedback loop in \eqref{eq:ema} relies on. However, this stability applies to the soft weights; hard top-$k$ selection may still introduce discontinuities when sample rankings change.

\begin{theorem}[Gradient bias]\label{thm:bias}
Let $S_t$ be the subset selected by Algorithm~\ref{alg:otcal}, $\nabla F(w)$ the mean gradient of the mini-batch, and $\nabla_{S_t}F(w)$ the mean gradient of $S_t$. Let $\mu_t$ and $\nu_t$ be the uniform distributions over the profiles of the batch and of $S_t$. Under Assumption~\ref{ass:feat},
\begin{align*}
\norm{\nabla F(w) - \nabla_{S_t}F(w)}_2 &\leq C_f\,W_1(\mu_t,\nu_t) + r\\
&\leq C_f\sqrt{d}\,\Big(1 - \frac{k}{N}\Big) + r,
\end{align*}
where $W_1$ is the 1-Wasserstein distance in profile space.
\end{theorem}
\begin{IEEEproof}
See Appendix, part 2.
\end{IEEEproof}

The bound has two parts. The first is the price of the curriculum: OT-CAL deliberately selects samples whose profiles differ from the batch, and the bias grows with that difference, scaled by how closely profile distances track gradient distances ($C_f$); it vanishes as $k\to N$. The second, $r$, is the part of gradient variation that the profile cannot explain. Lower $C_f$ and $r$ tighten the bound, which motivates the learned 8-D profile in Section~\ref{sec:profile}.

\begin{proposition}[Curriculum contraction]\label{prop:contr}
Write the update \eqref{eq:ema} as $c_{t+1} = \Phi(c_t)$ with $\Phi(c) = (1-\eta_c)c + \eta_c\fbar(c)$. If $\fbar$ is $L_{\fbar}$-Lipschitz in $c$, then $\Phi$ is Lipschitz with constant $\rho \leq 1 - \eta_c(1 - L_{\fbar})$, and is a contraction whenever $L_{\fbar} < 1$. Perturbations of the curriculum then decay geometrically at rate $\rho$, with $\rho\to 1-\eta_c = 0.95$ as $L_{\fbar}\to 0$.
\end{proposition}

For the soft-weighted mean $\fbar(c) = \sum_i q_i(c) f_i$, Proposition~\ref{prop:stab} gives $L_{\fbar}\leq 2d/\epst$ (Appendix, part 3). This worst-case bound is far above 1 for our $\epst = 0.11$, so Proposition~\ref{prop:stab} does not by itself certify contraction; Proposition~\ref{prop:contr} states when the loop is guaranteed to settle. Top-$k$ truncation in \eqref{eq:ema} makes $\fbar$ only piecewise Lipschitz, so the statement describes the soft surrogate between threshold crossings.

\begin{corollary}[Biased SGD]\label{cor:sgd}
Under $L$-smoothness and bounded gradient variance, SGD whose gradient bias is at most $B = C_f W_1(\mu_t,\nu_t) + r$ (Theorem~\ref{thm:bias}) reaches, at the usual $O(1/\sqrt{n})$ rate after $n$ iterations, a region where $\mathbb{E}\norm{\nabla F(w)}_2^2 = O(B^2)$ \cite{ghadimi2013,ajalloeian2020}.
\end{corollary}

Existing biased-SGD analyses treat the bias as fixed or uniformly bounded. Combining Corollary~\ref{cor:sgd} with Theorem~\ref{thm:bias} links the neighborhood size directly to feature quality $C_f$, connecting feature design (Section~\ref{sec:profile}) to optimization.

\section{Experiments}\label{sec:experiments}

We evaluate OT-CAL as an adaptive data-selection method from two angles: whether it improves the accuracy--efficiency tradeoff of online batch selection (Section~\ref{sec:tradeoff}), and which of its components are responsible (Section~\ref{sec:ablation}).Section~\ref{sec:dynamics} then describes how the selection evolves during training.

\subsection{Experimental Setup}\label{sec:setup}

\emph{Datasets and models.} CIFAR-100 \cite{krizhevsky2009} (50k training and 10k test images, $32\times32$, 100 classes) with a CIFAR-style ResNet-18 \cite{he2016} ($3\times3$ stem, no max-pooling). Tiny-ImageNet \cite{le2015} (100k training images, $64\times64$, 200 classes; we report on the 10k validation set because test labels are not public) with a ResNet-50 trained from scratch with the same CIFAR-style stem, and with a ViT-B/16 \cite{dosovitskiy2021} initialized from the AugReg ImageNet-21k checkpoint \cite{steiner2022} fine-tuned on ImageNet-1k (timm: \path{vit_base_patch16_224.augreg2_in21k_ft_in1k}) and fine-tuned at the native $64\times64$ resolution ($4\times4$ patches, interpolated position embeddings). Because ImageNet pretraining overlaps the images of Tiny-ImageNet, the ViT results measure fine-tuning efficiency rather than learning from scratch. COCO 2017 \cite{lin2014} detection (118k train2017 and 5k val2017 images) with DETR-R50 \cite{carion2020}.

\emph{Training.} The ResNets use SGD (momentum 0.9) with initial learning rate 0.1, decayed by $10\times$ at 75\% and 90\% of training, weight decay $5\times10^{-4}$, batch size 128, and no warm-up, for 250 epochs on CIFAR-100 and 300 epochs on Tiny-ImageNet, with random cropping (padding 4 on CIFAR-100, 8 on Tiny-ImageNet) and horizontal flipping. ViT-B/16 is fine-tuned for 30 epochs with SGD (momentum 0.9), learning rate 0.01 with a cosine schedule and 500-step warm-up, no weight decay, gradient clipping at 1.0, and batch size 128. DETR follows the original recipe \cite{carion2020}: shorter image side 800 (longer $\leq 1333$), ImageNet-pretrained backbone, AdamW with learning rate $10^{-4}$ ($10^{-5}$ for the backbone), weight decay $10^{-4}$, gradient clipping at 0.1, and 500 epochs with a $10\times$ decay at epoch 400. All methods train for the same number of epochs; subset methods reduce only the backward pass.

\emph{OT-CAL settings.} Sampling ratio $\gamma = 0.5$ unless stated, $\varepsilon = 0.01$ with 50 Sinkhorn iterations, $\tau = 0.1$, $\eta_c = 0.05$, and $\sigma = 0.05$, identical across datasets. Table~\ref{tab:main} uses the learned 8-D profile; ablations state the profile they use.

\emph{Baselines.} Full-data training (upper bound); AdaSelection \cite{zhang2023}; selective backpropagation (SB) \cite{jiang2019} with selectivity $\beta = 1$, which keeps a sample with probability equal to its loss percentile and so keeps half the batch ($\beta = 3$ for $\gamma = 0.25$); DataMap-based selection (DM) \cite{swayamdipta2020}, computed online from a per-sample EMA of the true-class probability (confidence) and its standard deviation across epochs (variability), keeping the $k$ most variable samples in each batch; and GRAD-MATCH \cite{killamsetty2021gradmatch} (per-batch variant, last-layer gradients, subset re-selected every 20 epochs) on CIFAR-100 only.

\emph{Protocol.} We report mean $\pm$ standard deviation over 10 seeds on CIFAR-100 and 5 seeds in the other settings, and compare OT-CAL with AdaSelection, SB, and DM in each setting using paired $t$-tests with Holm correction over the 12 comparisons. Accuracy is final-epoch top-1 accuracy on the CIFAR-100 test set and the Tiny-ImageNet validation set; COCO reports box mAP on val2017. Speedup is the ratio of full-data to method end-to-end wall-clock training time (data loading, forward, selection, backward, and optimizer step) for the same number of epochs on the same hardware. Experiments run on NVIDIA A100 80GB GPUs on the Purdue Gilbreth cluster.

\subsection{Accuracy--Efficiency Tradeoff}\label{sec:tradeoff}

\begin{table}[t]
\caption{Accuracy (\%) and Speedup at a 50\% Sampling Ratio}
\label{tab:main}
\centering
\scriptsize
\setlength{\tabcolsep}{3pt}
\begin{tabular}{@{}lccccccc@{}}
\toprule
\textbf{Dataset / model} & \textbf{Full} & \textbf{AdaSel} & \textbf{SB} & \textbf{DM} & \textbf{OT-CAL} & $\boldsymbol{\Delta}$ & \textbf{Speedup}\\
\midrule
CIFAR-100 / R18 & 77.5 & 74.2 & 74.5 & 74.4 & \textbf{75.8}$\pm$0.3 & +1.6 & $1.24\times^{\dagger}$\\
Tiny-IN / R50   & 64.3 & 62.8 & 63.0 & 62.9 & \textbf{64.9}$\pm$0.4 & +2.1 & $1.22\times$\\
Tiny-IN / ViT-B & 68.5 & 67.1 & 67.4 & 67.3 & \textbf{68.7}$\pm$0.3 & +1.6 & $1.18\times$\\
COCO / DETR     & 42.0 & 41.6 & 41.8 & 41.7 & \textbf{42.9}$\pm$0.3 & +1.3 & $1.18\times$\\
\bottomrule
\multicolumn{8}{@{}p{0.97\columnwidth}@{}}{\scriptsize Mean $\pm$ std over 10 seeds (CIFAR-100) or 5 seeds (others); std shown for OT-CAL. COCO reports box mAP. $\Delta$ is relative to AdaSelection (points); all differences significant ($p_{\mathrm{adj}} < 0.05$, Holm). SB: selective backpropagation \cite{jiang2019}; DM: DataMap \cite{swayamdipta2020}. OT-CAL uses learned 8-D profiles. $^{\dagger}$CIFAR-100 speedup measured with the fixed 5-D profile (Fig.~\ref{fig:speed}); the learned 8-D variant adds a small MLP cost and has not been re-timed.}
\end{tabular}
\end{table}

\emph{Accuracy.} At a 50\% sampling ratio, OT-CAL outperforms AdaSelection by 1.3--2.1 points in all four settings (Table~\ref{tab:main}): by 1.6 points on CIFAR-100 with ResNet-18, by 2.1 and 1.6 points on Tiny-ImageNet with ResNet-50 and ViT-B/16, and by 1.3 mAP on COCO. It also outperforms selective backpropagation and DataMap selection everywhere. With ResNet-50 on Tiny-ImageNet and on COCO, OT-CAL slightly exceeds full-data training at the same number of epochs (64.9\% vs.\ 64.3\%; 42.9 vs.\ 42.0 mAP), and with ViT-B/16 it matches full-data training (68.7\% vs.\ 68.5\%, within one standard deviation). We attribute this, tentatively, to a regularizing effect: these are the settings most prone to overfitting (ResNet-50 trained from scratch on 100k images, and DETR's long schedule), and down-weighting samples the model has already learned limits it. On CIFAR-100 a gap of 1.7 points remains.

\emph{Training efficiency.} OT-CAL trains $1.18$--$1.24\times$ faster than full-data training in wall-clock time. Because the backward pass costs about twice the forward pass and the forward pass still runs on every sample, halving the backward pass bounds the speedup at about $1.5\times$; the measured speedups sit below this bound. Fig.~\ref{fig:speed} shows the resulting accuracy--efficiency tradeoff on CIFAR-100, using the fixed 5-D informativeness profile. At $\gamma = 0.5$, OT-CAL is more accurate than GRAD-MATCH and AdaSelection at similar speedup. At $\gamma = 0.25$, it comes within 0.2 points of AdaSelection at $\gamma = 0.5$ while running faster ($1.40\times$ vs.\ $1.30\times$, against a $2\times$ bound), so lowering $\gamma$ trades accuracy for speed gradually.

\begin{figure}[t]
\centerline{\includegraphics[width=\columnwidth]{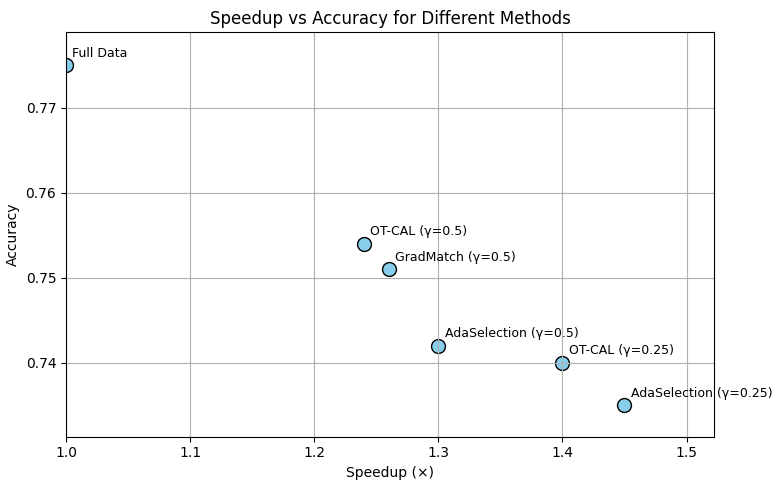}}
\caption{Accuracy--efficiency tradeoff on CIFAR-100 (ResNet-18) for sampling ratios $\gamma\in\{0.25, 0.5\}$; upper left is better. OT-CAL points use the fixed 5-D profile (75.4\% at $1.24\times$ for $\gamma = 0.5$); the learned 8-D result of Table~\ref{tab:main} (75.8\%) is not plotted.}
\label{fig:speed}
\end{figure}

\emph{Computational overhead.} On CIFAR-100 with batch size 128, computing the profiles takes about 0.3~ms and the Sinkhorn step (50 iterations) about 1.2~ms per mini-batch, together about 1.5~ms of a 32.5~ms OT-CAL training step on the A100, or about 5\% (a full-data step takes about 40~ms). For reference, a full-data step on a batch of 128 costs about 427~GFLOPs and an OT-CAL step at $\gamma = 0.5$ about 284~GFLOPs; at the A100's FP32 peak (19.5~TFLOPS) these would take 21.9~ms and 14.6~ms, so the measured steps run at about twice the FP32 compute floor. Profiles cost $O(N)$, and each Sinkhorn iteration costs $O(NM)$ for $M$ target points, which is $O(N)$ for the single target used here; the final top-$k$ costs $O(N\log N)$. Memory overhead is one loss-EMA scalar per training example plus the current batch's profiles. Gradient-matching coresets are far more expensive: storing per-sample last-layer gradients for ResNet-50 on Tiny-ImageNet ($100\text{k samples}\times 2048\times 200$ values, about 164~GB in 32-bit floats) is impractical, so we run GRAD-MATCH only on CIFAR-100.

\subsection{Component Ablation}\label{sec:ablation}

We separate the contributions of the feature representation, the OT selection mechanism, and the individual features, on CIFAR-100 with ResNet-18.

\begin{table}[t]
\caption{Feature Representation vs.\ Selection Mechanism (CIFAR-100, ResNet-18)}
\label{tab:repsel}
\centering
\footnotesize
\begin{tabular}{@{}llccc@{}}
\toprule
\textbf{Profile} & \textbf{Selector} & \textbf{Acc.\ (\%)} & $\boldsymbol{R^2}$ & \textbf{Hamming}\\
\midrule
Fixed 5-D   & Top-$k$ & 75.1 & 0.68 & n/a\\
Fixed 5-D   & OT      & 75.4 & 0.68 & n/a\\
Learned 8-D & Top-$k$ & 75.6 & 0.83 & 0.88\\
Learned 8-D & OT      & \textbf{75.8} & 0.83 & \textbf{0.94}\\
\bottomrule
\multicolumn{5}{@{}p{0.97\columnwidth}@{}}{\scriptsize $R^2$: linear-regression fit of last-layer per-sample gradient norms from the profile, on a random 10k training subset, averaged over checkpoints taken every 25 epochs. Hamming: selection similarity on a fixed probe set of 1,024 training samples, $1 - \mathrm{mean}\,|m_t - m_{t+1}|$ between the selection masks of consecutive iterations, averaged over training; not logged for the 5-D runs (n/a). Mean over 10 seeds.}
\end{tabular}
\end{table}

\emph{Representation and selection.} Table~\ref{tab:repsel} shows that both parts contribute. OT-based selection outperforms direct top-$k$ selection for both representations: by 0.3 points with the fixed 5-D profile (75.1\% to 75.4\%) and by 0.2 points with the learned 8-D profile (75.6\% to 75.8\%), where it also raises selection stability from 0.88 to 0.94. The learned representation improves accuracy under both selectors, by 0.5 points with top-$k$ and 0.4 points with OT, and raises the feature--gradient $R^2$ from 0.68 to 0.83, in line with the role of feature quality in Theorem~\ref{thm:bias}. The gains from OT are modest but consistent.

\begin{table}[t]
\caption{Per-Feature Ablation (Fixed 5-D, OT Selection; CIFAR-100, ResNet-18)}
\label{tab:feat}
\centering
\footnotesize
\begin{tabular}{@{}lccc@{}}
\toprule
\textbf{Feature removed} & \textbf{Acc.\ (\%)} & \textbf{Drop (pp)} & \textbf{Share}\\
\midrule
None (full profile) & 75.4 & --- & ---\\
Learnability ($L$) & 74.1 & $-1.3$ & 41\%\\
Variety ($V$) & 74.6 & $-0.8$ & 25\%\\
Uncertainty ($U$) & 74.8 & $-0.6$ & 19\%\\
Gradient proxy ($G$) & 75.1 & $-0.3$ & 9\%\\
Difficulty ($D$) & 75.2 & $-0.2$ & 6\%\\
\bottomrule
\end{tabular}
\end{table}

\emph{Individual features.} Table~\ref{tab:feat} removes one feature at a time, retraining each variant from scratch (10 seeds). Every removal lowers accuracy. Learnability matters most ($-1.3$ points, 41\% of the summed drops), followed by variety ($-0.8$, 25\%), uncertainty ($-0.6$, 19\%), gradient proxy ($-0.3$, 9\%), and difficulty ($-0.2$, 6\%). The trajectory-based signal contributes more than the magnitude-based ones, consistent with findings on example forgetting and learnability \cite{toneva2019,mindermann2022}.

\subsection{Adaptive Curriculum Dynamics}\label{sec:dynamics}

\begin{figure}[t]
\centerline{\includegraphics[width=\columnwidth]{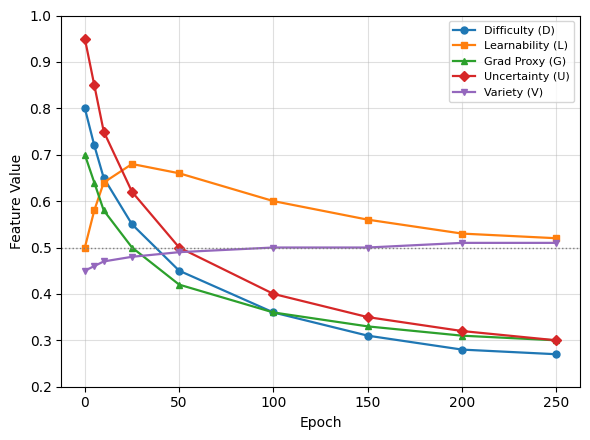}}
\caption{Components of the curriculum vector $c_t$ for the fixed 5-D informativeness profile, derived from the feature definitions in Section~\ref{sec:profile}. The dotted line marks $L=0.5$, i.e., no change in loss.}
\label{fig:curriculum}
\end{figure}

Fig.~\ref{fig:curriculum} sketches the behavior we expect from the feature definitions; it is a schematic, not a measured trajectory. At initialization the network predicts almost uniformly, so uncertainty and difficulty are near their maximum, and learnability is exactly 0.5 because the loss EMA starts at each sample's first loss. As loss falls, difficulty, uncertainty, and the gradient proxy decrease, while learnability rises above 0.5 and then relaxes back toward 0.5 once losses stop changing; variety changes only slightly. The curriculum therefore shifts from hard, uncertain samples toward samples whose loss is still falling, without a phase schedule. We have not logged $c_t$ in our runs, so this figure does not show that this shift causes the accuracy gains in Table~\ref{tab:main}.

\section{Limitations}

OT-CAL skips backward passes but still runs the forward pass on every sample, which caps the attainable speedup at a given $\gamma$. On CIFAR-100 it does not reach full-data accuracy. The variety feature relies on class labels, so self-supervised pretraining would need a different diversity signal. The guarantees rest on Assumption~\ref{ass:feat}, which the 5-D profile satisfies only partially ($R^2 = 0.68$, Table~\ref{tab:repsel}); the gradient-bias bound grows with how far the selected profiles are from the batch's; and our worst-case bound does not certify the contraction condition of Proposition~\ref{prop:contr} for the $\epst$ we use. The curriculum dynamics in Section~\ref{sec:dynamics} are schematic, not measured, and do not establish a causal link to accuracy. Finally, we have not yet evaluated full ImageNet-1K training or fine-tuning of foundation models.

\section{Conclusion}

OT-CAL merges curriculum learning and online subset selection by matching mini-batch informativeness profiles to an evolving curriculum target with entropic optimal transport. The target is updated by an EMA of the selected profiles, forming a closed loop whose selections shift from hard samples to improving samples to a class-balanced mix without a manual schedule, at $O(N\log N)$ per-batch cost and without per-sample gradients. The analysis links transport-plan stability, feature--gradient alignment, and biased-SGD convergence. Experiments on CIFAR-100, Tiny-ImageNet, and COCO show gains of 1.3--2.1 points over AdaSelection at $1.18$--$1.24\times$ speedup, with learnability and variety as the dominant features and modest but consistent gains from OT selection over top-$k$ ranking. Extending OT-CAL to foundation-model fine-tuning and to label-free diversity signals is a natural next step.

\section*{Appendix: Proofs}

\subsubsection{Proposition~\ref{prop:stab}}
Write $q = \softmax(-C/\epst)$ as in \eqref{eq:closed} and let $\delta = C' - C$. Along $C(s) = C + s\delta$, $s\in[0,1]$, the softmax derivative is $\mathrm{d}q_i/\mathrm{d}s = -(1/\epst)\,q_i(\delta_i - \bar\delta)$, with $\bar\delta = \sum_j q_j\delta_j$. Since $\bar\delta$ lies between $\min_j\delta_j$ and $\max_j\delta_j$,
\[
\norm{\mathrm{d}q/\mathrm{d}s}_1 \leq \frac{1}{\epst}\big(\max_j\delta_j - \min_j\delta_j\big) \leq \frac{2}{\epst}\norm{\delta}_\infty;
\]
integrating over $s$ gives $\norm{q(C') - q(C)}_1 \leq (2/\epst)\norm{C' - C}_\infty$. For the costs, put $u = f_i - T$ and $v = f'_i - T'$. Then
\begin{align*}
|C_i - C'_i| &= |(u-v)^\top(u+v)| \leq \norm{u-v}_2\norm{u+v}_2\\
&\leq 2\sqrt{d}\,\big(\norm{f_i - f'_i}_2 + \norm{T - T'}_2\big),
\end{align*}
because every vector lies in $[0,1]^d$, so $\norm{u}_2, \norm{v}_2 \leq \sqrt{d}$. Using $\max_i\norm{f_i - f'_i}_2 \leq \norm{F - F'}_F$ and combining the two bounds proves the claim with $L_{\mathrm{OT}} = 4\sqrt{d}/\epst$.\hfill$\blacksquare$

\subsubsection{Theorem~\ref{thm:bias}}
Let $\pi$ be any coupling of $\mu_t$ and $\nu_t$: $\pi_{ij}\geq 0$ with $\sum_j\pi_{ij} = 1/N$ for every batch sample $i$ and $\sum_i \pi_{ij} = 1/k$ for every $j\in S_t$. Then $\nabla F(w) - \nabla_{S_t}F(w) = \sum_{i,j}\pi_{ij}(\nabla\ell_i - \nabla\ell_j)$, so by the triangle inequality and Assumption~\ref{ass:feat},
\[
\norm{\nabla F(w) - \nabla_{S_t}F(w)}_2 \leq C_f\sum_{i,j}\pi_{ij}\norm{f_i - f_j}_2 + r.
\]
Minimizing over $\pi$ gives the first inequality. For the second, keep the mass that $\mu_t$ and $\nu_t$ share in place: each sample in $S_t$ has mass $1/N$ under $\mu_t$ and $1/k$ under $\nu_t$, so a fraction $1 - k/N$ of the mass moves, by at most the diameter $\sqrt{d}$ of $[0,1]^d$. Hence $W_1(\mu_t,\nu_t) \leq \sqrt{d}\,(1 - k/N)$.\hfill$\blacksquare$

\subsubsection{Proposition~\ref{prop:contr} and the bound on $L_{\fbar}$}
$\Phi(c) - \Phi(c') = (1-\eta_c)(c - c') + \eta_c(\fbar(c) - \fbar(c'))$, so $\norm{\Phi(c) - \Phi(c')}_2 \leq (1 - \eta_c + \eta_c L_{\fbar})\norm{c - c'}_2 = \rho\norm{c - c'}_2$. If $L_{\fbar} < 1$, then $\rho < 1$ and, by Banach's fixed-point theorem, $\Phi$ has a unique fixed point $c^{*}$ with $\norm{c_t - c^{*}}_2 \leq \rho^t\norm{c_0 - c^{*}}_2$; a perturbation of the curriculum shrinks by at least the factor $\rho$ per update. The argument holds for a fixed set of profiles; it does not account for the drift of the profiles during training. For the soft-weighted mean $\fbar(c) = \sum_i q_i(c) f_i$, the weights sum to one, so for the cube center $m = (\tfrac12,\dots,\tfrac12)$,
\begin{align*}
\norm{\fbar(c) - \fbar(c')}_2 &= \Big\lVert\sum_i (q_i - q'_i)(f_i - m)\Big\rVert_2\\
&\leq \frac{\sqrt{d}}{2}\norm{q - q'}_1 \leq \frac{2d}{\epst}\norm{c - c'}_2,
\end{align*}
by Proposition~\ref{prop:stab} with the same noise draw (so $T - T' = c - c'$).\hfill$\blacksquare$


\end{document}